\documentclass[11pt,english,reqno]{amsart}

\pdfoutput=1

\usepackage[T1]{fontenc}
\usepackage[utf8]{inputenc}
\usepackage{geometry}
\usepackage{babel}
\usepackage{mathtools}
\usepackage{amsmath,amsthm,amssymb}
\usepackage{esint}
\usepackage{microtype}
\usepackage[unicode=true, pdfusetitle, bookmarks=true, bookmarksnumbered=false, bookmarksopen=false, breaklinks=true, pdfborder={0 0 0}, pdfborderstyle={}, backref=false, colorlinks=true, linkcolor=blue, citecolor=blue]{hyperref}
\usepackage[capitalise]{cleveref}
\crefformat{equation}{(#2#1#3)}
\crefname{prop}{Proposition}{Propositions}
\crefname{thm}{Theorem}{Theorems}
\crefname{sec}{Section}{Sections}

\usepackage{tikz}
\usepackage{xcolor}
\definecolor{darkgreen}{rgb}{0,0.5,0}
\definecolor{darkblue}{rgb}{0,0,0.7}
\definecolor{darkred}{rgb}{0.9,0.1,0.1}

\numberwithin{figure}{section}
\numberwithin{equation}{section}
\numberwithin{table}{section}
\theoremstyle{plain}
\newtheorem{thm}{\protect\theoremname}[section]
\theoremstyle{definition}
\newtheorem{defn}[thm]{\protect\definitionname}
\theoremstyle{plain}
\newtheorem{prop}[thm]{\protect\propositionname}
\theoremstyle{remark}
\newtheorem{rem}[thm]{\protect\remarkname}

\providecommand{\definitionname}{Definition}
\providecommand{\propositionname}{Proposition}
\providecommand{\remarkname}{Remark}
\providecommand{\theoremname}{Theorem}

\newcommand{\e}{\mathrm{e}}%
\newcommand{\R}{\mathbf{R}}%
\newcommand{\Rd}{{\mathbf{R}^d}}
\newcommand{\dist}{\operatorname{dist}}%
\newcommand{\supp}{\operatorname{supp}}%
\newcommand{\diam}{\operatorname{diam}}%
\newcommand{\esssup}{\operatorname*{ess\,sup}}%
\newcommand{\id}{\operatorname{id}}%
\renewcommand{\le}{\leqslant}
\renewcommand{\ge}{\geqslant}

\renewcommand{\subset}{\subseteq}
\newcommand{\Ll}{\left}
\newcommand{\Rr}{\right}
\newcommand{\dr}{\partial}
\newcommand{\dif}{\mathrm{d}}%
\newcommand{\mcl}{\mathcal}

\newcommand{\de}{\delta}
\newcommand{\ga}{\gamma}
\newcommand{\eps}{\varepsilon}%
\newcommand{\ep}{\varepsilon}%
\newcommand{\la}{\lambda}
\renewcommand{\d}{\mathrm{d}}
\renewcommand{\tilde}{\widetilde}

\renewcommand{\bar}{\overline}

\newcommand{\BV}{\operatorname{BV}}
\newcommand{\centroid}[1]{\operatorname{cent}_\mu}

\newcommand{\mres}{\mathbin{\vrule height 1.4ex depth 0pt width
0.13ex\vrule height 0.13ex depth 0pt width 1ex}}

\newcommand{\mressmall}{\mathbin{\vrule height 1ex depth 0pt width
0.1ex\vrule height 0.1ex depth 0pt width 0.7ex}}

\begin{document}
\title{Local versions of sum-of-norms clustering}

\author[A. Dunlap]{Alexander Dunlap}
\address[A. Dunlap]{Courant Institute of Mathematical Sciences, New York University, New York, NY 10012 USA}
\email{alexander.dunlap@cims.nyu.edu}

\author[J.-C. Mourrat]{Jean-Christophe Mourrat}
\address[J.-C. Mourrat]{Courant Institute of Mathematical Sciences, New York University, New York, NY 10012 USA; CNRS, Ecole Normale Sup\'erieure de Lyon, Lyon, France}
\email{jean-christophe.mourrat@ens-lyon.fr}

\begin{abstract}
Sum-of-norms clustering is a convex optimization problem whose solution can be used for the clustering of multivariate data. We propose and study a localized version of this method, and show in particular that it can separate arbitrarily close balls in the stochastic ball model. More precisely, we prove a quantitative bound on the error incurred in the clustering of disjoint connected sets. Our bound is expressed in terms of the number of datapoints and the localization length of the functional.
\end{abstract}

\maketitle

\section{Introduction}\label{sec:intro}

\subsection{Context and informal description of main result}

Let $x_1,\ldots,x_{N}\in\mathbf{R}^{d}$ ($d\in\mathbf{N}$)
be a collection of points, which we think of as a dataset. We consider
the clustering problem, which is to find a partition of $\{x_1,\ldots,x_{N}\}$
that collects close-together points into the same element of the partition.
The problem of %
\emph{$K$-means clustering} is %
to identify a global minimizer of the functional
\begin{equation}
(y_1,\ldots,y_{N})\mapsto\frac{1}{N}\sum_{n=1}^{N}|y_{n}-x_{n}|^2,
\label{e.k-means}
\end{equation}
over all $(y_1,\ldots, y_N) \in (\R^d)^N$ such %
that the cardinality of the set $\{y_1,\ldots, y_N\}$ is at most~$K$. This minimization problem is known to be NP-hard in general, even when restricted to $K = 2$ \cite{aloise09} or $d = 2$ \cite{MNV09}. Practitioners typically resort to iterative search algorithms such as Lloyd's algorithm and its refinements \cite{lloyd1982least, vassilvitskii2006k}, which at least  identify local minimizers of \cref{e.k-means}. However, these methods are known to perform poorly in some cases, as will be discussed further below.

In this paper, we focus our attention on the ``sum-of-norms clustering''
method (also known as ``convex clustering shrinkage'' or ``Clusterpath'')
introduced in \cite{PDBSDM05,HVBJ11,LOL11}. This method can be thought of as a convex relaxation of the $K$-means problem. It considers the minimizer of the convex functional
\begin{equation}
(y_1,\ldots,y_{N})\mapsto\frac{1}{N}\sum_{n=1}^{N}|y_{n}-x_{n}|^2+\frac{\lambda}{N^2}\sum_{m,n=1}^{N}w(|x_{m}-x_{n}|)|y_{m}-y_{n}|\label{eq:discrete-functional-general-weight}
\end{equation}
over $(y_1,\ldots,y_{N})\in(\Rd)^N$,
for some nonincreasing ``weight function'' $w$. (Typical choices include constant and exponential weights.) Here $|\cdot|$ denotes the Euclidean norm on $\R^d$. The point $y_{n}$ is thought
of as a ``representative point'' of the cluster to which $x_{n}$
belongs, and thus $x_{n}$ and $x_{m}$ are declared to be members of the same cluster if
$y_{n}=y_{m}$. The first term of~\cref{eq:discrete-functional-general-weight} is designed
to keep the representative point of a cluster close to the points
in that cluster (and so encouraging having many clusters), while the
second term (called the ``fusion term'') is designed to encourage
points to merge into fewer clusters, at least if they are close together
according to the weight function. The parameter $\lambda$ controls
the relative strength of these two effects.

The present work investigates an asymptotic regime of sum-of-norms
clustering as the number of datapoints becomes very large and the
weight $w$ is simultaneously scaled in a careful way. In order to do so, it is useful to specify a more explicit model for the dataset. We assume that the datapoints $x_1, \ldots, x_N$ are independent and identically distributed. Their common law $\mu$, a probability measure on $\Rd$, is supported on the  union of disjoint closed sets $\bar U_1, \ldots, \bar U_L$. %
These sets are not known to the practitioner. We would like  $x_i$ and $x_j$ to be in the same cluster 
if and only if they lie in the same set $\bar U_\ell$ for some $\ell \in \{1,\ldots, L\}$, and so we seek a clustering algorithm that can guarantee this in the limit as $N\to \infty$. 

The weight function we choose is $w(r) := \gamma^{d+1} e^{-\gamma r}$, where $\gamma > 0$ is a parameter that can be tuned with the number of datapoints $N$. Roughly speaking, our main result states that, under modest assumptions, if we choose $\lambda$ above a critical threshold not depending on $N$, and also choose $\gamma \simeq N^{3/(4d)}$, then in some mean-square sense, each point $x_n \in \bar U_\ell$ will be associated with a ``representative point'' $y_n$ that is at distance of about $N^{-1/(8d)}$ from the centroid of the set $\bar U_\ell$ as $N\to\infty$. In particular, the clustering of the dataset is successful in the mean-square sense. The technical assumptions we need are that each set~$\bar U_\ell$ is ``effectively'' star-shaped (see Definition~\ref{def:effectivelystarshaped} below), that the measure $\mu$ has a density with respect to the Lebesgue measure, and that this density is Lipschitz and bounded away from zero on its support. As an illustration, we can take $\mu$ to be the uniform measure on the union of the sets depicted in \cref{fig:clusters} below. The condition that the clusters be effectively star-shaped is a nontrivial geometric restriction, although it does not seem to be fundamental. See \cref{rem:starshapedthing} below for a weaker but more complicated sufficient condition, and further discussion.

Our result applies in particular to the case in which $\mu$ is the uniform measure on the union of disjoint balls. One of the strengths of our result is that these balls, or more generally the sets $\bar U_1, \ldots, \bar U_L$, can be chosen arbitrarily close to one another, as long as they do not touch. (However, we expect that the required number of datapoints $N$ will grow as the balls are brought closer to each other.) Another important feature is that we allow for sets $\bar U_1,\ldots, \bar U_L$ that may be non-convex, as long as they are effectively star-shaped. Moreover, our result covers situations in which the convex hulls of the clusters intersect. 

The unweighted version of the sum-of-norms clustering method, i.e.\ the case $w \equiv 1$, does not share any of these features. Indeed, the unweighted method fails to recover the clusters of datapoints sampled independently from two disjoint balls if
the balls are too close together, as we showed in~\cite{DM21}. Moreover, the unweighted algorithm must output clusters that are contained in disjoint balls (see \cite[Theorem~3]{NM21} or \cite[Proposition~1.8]{DM21}), and in particular, it cannot separate two clusters unless their convex hulls are disjoint. 

Popular alternative clustering methods such as Lloyd's algorithm and its refinements \cite{lloyd1982least, vassilvitskii2006k} are also known to have important limitations. In \cite[Appendix E]{awasthi2015relax}, the authors exhibit explicit examples of configurations of disjoint balls $\bar U_1, \ldots, \bar U_L$ of equal radius such that if  the measure $\mu$ is the uniform probability measure on the union of these balls, then the probability that Lloyd's algorithm successfully clusters the dataset is at most $(1-\frac 2 9)^{L/3}$. They also construct similar examples for which a refined method called ``kmeans++'' also fails to successfully cluster the dataset with a probability that can be made arbitrarily close to $1$. 

Other convex relaxations of the $K$-means problem have been explored, but we are not aware of theoretical guarantees that would cover the case in which two clusters can be taken arbitrarily close to one another. Possibly the simplest way to ask the question is to consider the ``stochastic ball model'' \cite{NW15}: we assume that the datapoints are sampled independently according to the uniform measure on the union of two disjoint balls of unit radius. In this setting, the method explored in \cite{awasthi2015relax} is guaranteed to recover the clusters provided that the distance between the two ball centers is above $2 \sqrt{2} (1 + d^{-1/2})$. (See also \cite{del2021k} for the related problem of $K$-medians clustering.) Another convex relaxation of $K$-means clustering is explored in \cite{de2020ratio}: for the stochastic ball model, that method successfully clusters the dataset provided that the distance between the centers of the balls is above $1 + \sqrt{3}$. 

Several earlier works have explored the theoretical properties of sum-of-norms clustering. The unweighted method ($w \equiv 1$) was shown to separate cube-shaped clusters provided that they are sufficiently far away in \cite{ZXLY14}; for the case of two cubes of side-length $2$ and equal number of datapoints falling in each cube, the criterion requires that the minimal distance between two points in each cube be at least $6 \sqrt{d}$. More general conditions are derived in~\cite{PDJB17} (see in particular part~2 of Theorem~1) that imply the successful recovery of the clusters for the stochastic ball model if the distance between the ball centers is larger than $4$.  %
These results were refined and extended to the case of arbitrary weights in \cite{STY21}. The problem of separating mixtures of Gaussian random variables has been considered in \cite{TW15, PDJB17, JVZ19}, and algorithmic aspects were explored in \cite{PDBSDM05, HVBJ11, chi2015splitting, CGR17, JV20}. Several works have stressed the apparent advantages of using non-constant weights in sum-of-norms clustering \cite{HVBJ11, chi2015splitting, CGR17, NM21}.

\subsection{Precise statement and proof strategy}

Following our
previous work \cite{DM21}, for the purposes of mathematical analysis
we consider the somewhat more general problem of clustering of measures.
For a Borel measure $\mu$ on $\mathbf{R}^{d}$ of compact support, we abbreviate  $L^2(\mu) \coloneqq L^2(\R^d,\mu;\R)$ and $(L^2(\mu))^d \simeq L^2(\R^d,\mu; \R^d)$ to denote the Lebesgue spaces of $\mu$-square-integrable functions from $\R^d$ to $\R$ and $\R^d$ to $\R^d$ respectively. (We recall that these spaces identify functions that only disagree on a set of $\mu$-measure zero.) We define the functional $J_{\mu,\lambda,\gamma}:(L^2(\mu))^d\to\mathbf{R}$ by
\begin{equation}
J_{\mu,\lambda,\gamma}(u)\coloneqq\int|u(x)-x|^2\,\dif\mu(x)+\lambda\gamma^{d+1}\iint\e^{-\gamma|x-y|}|u(x)-u(y)|\,\dif\mu(x)\,\dif\mu(y).\label{eq:Jdef}
\end{equation}
We note that \cref{eq:discrete-functional-general-weight} with $w(r)=\gamma^{d+1}\e^{-\gamma r}$ is obtained from \cref{eq:Jdef} by setting $\mu=\frac{1}{N}\sum_{n=1}^{N}\delta_{x_{n}}$. The map $x\mapsto u(x)$ is then the analogue of the map $x_n\mapsto y_n$ from points to cluster representative points.
We denote by $u_{\mu,\lambda,\gamma}$ the minimizer of $J_{\mu,\lambda,\gamma}$,
which exists and is unique because $J_{\mu,\lambda,\gamma}$ is coercive, uniformly
convex, and continuous on $(L^2(\mu))^d$. (See \cref{eq:uniformconvexity} below.) For every Borel set $U$ such that $\mu(U) > 0$, we let
\[\centroid{\mu}(U)\coloneqq \frac{1}{\mu(U)}\int_U x\,\dif \mu(x)\] be the $\mu$-centroid of $U$. We also write $a \vee b := \max(a,b)$, and define
\begin{equation}  
\label{e.def.dprime}
d' := 
\begin{cases}  %
\infty & \text{if } d = 1, \\
\frac 4 3 & \text{if } d = 2, \\
d & \text{if } d \ge 3.
\end{cases}
\end{equation}

Our main result considers a measure $\mu$ with support comprising a finite union of connected components, each with sufficiently regular boundary and satisfying a quantitative version of a ``star-shaped'' property. We also assume that $\mu$ is bounded below on its support, and is sufficiently regular on its support. We draw $N$ datapoints independently from $\mu$ and run our clustering algorithm on these datapoints. If $\gamma$ is chosen appropriately large depending on $N$, and $\lambda$ is fixed sufficiently large independent of $N$, then our clustering algorithm will recover the connected components of $\supp\mu$. Before stating our main result, we introduce the technical condition we need on the components of $\supp\mu$. %

\begin{defn}
\label{def:effectivelystarshaped}For $U$ a subset of $\mathbf{R}^{d}$
and $\eps>0$, let $U_{\eps}$ be the $\eps$-enlargement of $U$,
namely
\[
U_{\eps}\coloneqq\{x\in\mathbf{R}^{d}\mid\dist(x,U)\le\eps\}.
\]
We say that a domain $U$ is \emph{effectively star-shaped} if
there exists $x_{*}\in U$ and a constant $C_{*}<\infty$ such that
for every $\eps>0$ sufficiently small, the image of $U_{\eps}$ under
the mapping $x\mapsto x_{*}+(1-C_{*}\eps)(x-x_{*})$ is contained
in $U$.
\end{defn}
For example, any convex open set is effectively star-shaped (in which case $x_*$ can be chosen arbitrarily in the interior). Any effectively star-shaped set is star-shaped. An example of a set that is star-shaped but not effectively star-shaped is illustrated in \cref{fig:starshapedthing}. Now we can state our main theorem.
\begin{figure}
    \centering
    \begin{tikzpicture}[scale=2]
      \fill[color=lightgray] (1,0) -- (2,0) arc (0:-180:1) -- (-1,0) arc (180:0:1);
    \end{tikzpicture}
    \caption{A set that is star-shaped but not effectively star-shaped.}
    \label{fig:starshapedthing}
\end{figure}
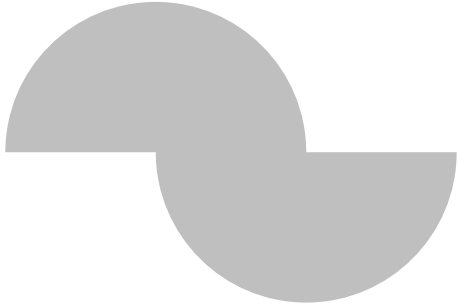

\begin{figure}
    \centering
    \includegraphics[scale=0.6]{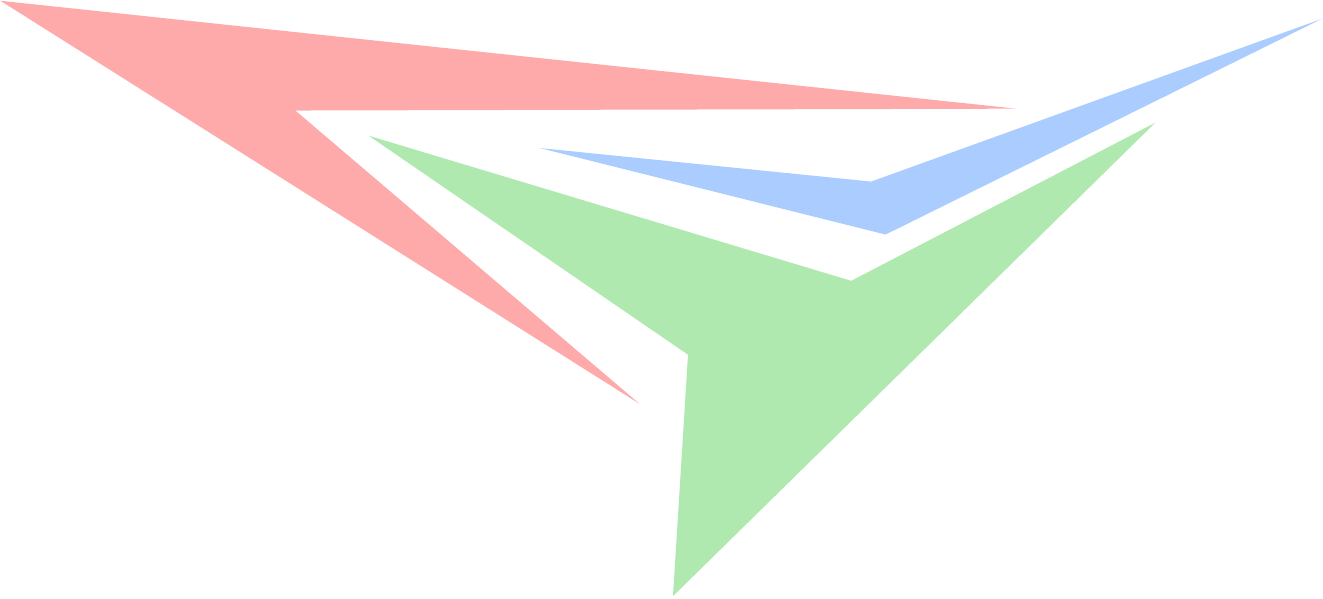}
    \caption{A set of three open sets $U_1,U_2,U_3$ satisfying the hypotheses of \cref{thm:maintheorem}.}
    \label{fig:clusters}
\end{figure}

\begin{thm}
\label{thm:maintheorem}Let $\mu$ be a probability measure on $\mathbf{R}^{d}$
such that $\supp\mu=\bigcup_{\ell=1}^{L}\overline{U_\ell}$, where $U_1,\ldots,U_L$
are bounded, effectively star-shaped %
open sets with Lipschitz boundaries, such that their closures $\overline{U_1},\ldots,\overline{U_L}$
are pairwise disjoint. Assume that~$\mu$ admits a density with respect
to the Lebesgue measure, and that this density is Lipschitz and bounded away from zero on $\supp\mu$. Then there exist $\lambda_{\mathrm{c}}, C<\infty$ such that for every $\lambda \ge \lambda_{\mathrm{c}}$, the following holds.
Let $(X_{n})_{n\in\mathbf{N}}$ be a sequence of independent random variables
with law~$\mu$, $N \ge 1$ be an integer, $\mu_{N}\coloneqq\frac{1}{N}\sum_{n=1}^{N}\delta_{X_{n}}$ be the empirical measure of the datapoints, and
\[
A_N^{(\ell)}\coloneqq\{n\in\{1,\ldots,N\}\mid X_{n}\in U_\ell\},\qquad \ell \in\{1,\ldots,L\}
\]
be the set of indices of datapoints in $U_\ell$.
For every $\ga \ge 1$, the mean-square error between the clustering algorithm and the centroids of the clusters is bounded as follows:
\begin{equation}
\begin{aligned}
\mathbf{E}&\left[\frac{1}{N}\sum_{\ell=1}^L\sum_{n\in A_{N}^{(\ell)}}|u_{\mu_{N},\lambda,\gamma}(X_{n})-\centroid{\mu}(U_\ell)|^2\right]\\
&\qquad\le C\left(\gamma N^{-1/(d\vee 2)}(\log N)^{1/d'}+(1+\lambda)\gamma^{-1/3}\right).\label{eq:thm1statement}
\end{aligned}
\end{equation}
\end{thm}

For $d \ge 2$, optimizing the right-hand side of \cref{eq:thm1statement} suggests the optimal choice $\gamma \simeq N^{3/(4d)}$, in which case the mean-square error is at most of the order of $N^{-1/(4d)}$, up to logarithmic corrections. We do not know if the estimate in \cref{eq:thm1statement} is sharp. If technical issues that arise near the boundary of the domains could be avoided, then we believe that we could replace the term $\gamma^{-1/3}$ in \cref{eq:thm1statement} by $\gamma^{-1/2}$; this in turn would suggest choosing $\gamma \simeq N^{2/(3d)}$, up to a logarithmic correction.

A similar result to \cref{thm:maintheorem} can be obtained if the weight $r\mapsto \e^{-\gamma r}$ is replaced by a truncated version $r\mapsto \e^{-\gamma r}\mathbf{1}_{r\le \omega}$ for an appropriate choice of $\omega$; see \cref{prop:truncate} below. This result essentially says that we can choose $\omega \simeq \ga^{-1}$, up to a logarithmic correction, without modifying the optimizer substantially. In the discrete setting, this reduces the number of pairs of points that need to be included in the sum that is the double integral in~\cref{eq:Jdef}, and thus may lead to improvements in computational efficiency. (See \cite{chi2015splitting} regarding efficient computational algorithms for sum-of-norms clustering, and in particular regarding the effect of the sparsity of the weights on the computational complexity.)
For instance, under the assumptions of \cref{thm:maintheorem} and with the choice of $\omega \simeq \ga^{-1} \simeq N^{-3/(4d)}$, a typical point only interacts with about $N^{1/4}$ points in its vicinity. Depending on the relative costs of computation versus the procurement of new datapoints, efficiency considerations may lead to a different choice of $\gamma$ than what would be suggested by the optimal accuracy considerations discussed in the previous paragraph. We do not further pursue the question of computational efficiency in the present paper.

While we did not keep track of this explicitly, one can check from the proof that the critical value $\lambda_c < \infty$ identified in \Cref{thm:maintheorem} does not change as the sets $\bar U_1, \ldots, \bar U_L$ are individually translated or rotated, provided that they remain pairwise disjoint. In particular, this constant does not depend on the minimal distance between the different data clusters. As a careful examination of the arguments below shows, one can also choose the constant $C < \infty$ in \Cref{thm:maintheorem} to be invariant under individual translations and rotations of the sets $\bar U_1, \ldots, \bar U_L$ that do not make them intersect each other, provided that we also require $\gamma$ to be sufficiently large. Roughly speaking, we would then require $\ga^{-1}$ to be larger than the minimal distance separating any pair of clusters, that is,
\begin{equation*}  %
\ga^{-1} \gtrsim \Delta := \min_{1 \le \ell \neq \ell'\le L}\dist(U_\ell,U_{\ell'}).
\end{equation*}
The precise condition is displayed in \eqref{eq:mindist} below. In particular, for $d \ge 2$, our approach would yield non-trivial information provided that the number of datapoints $N$ is much larger than $\Delta^{-d}$.

An important step in the proof of \cref{thm:maintheorem}, which is
also of independent interest, concerns the behavior of the functional $J_{u,\lambda,\gamma}$ as 
$\gamma$ is taken to infinity. The factor $\gamma^{d+1}$
in~\cref{eq:Jdef} was chosen so that $J_{u,\lambda,\gamma}$ would converge to a limiting functional as $\gamma\to\infty$, %
under appropriate conditions on $\mu$. Let~$U$ be a bounded open subset of $\mathbf{R}^{d}$ and suppose
that $\supp\mu=\overline{U}$. Suppose furthermore that $\mu$ is
absolutely continuous with respect to the Lebesgue measure on $\overline{U}$,
with density $\rho\in\mathcal{C}(\overline{U})$ bounded away from zero on $\bar U$. We denote by $\BV(U)$ the space of functions of bounded variation on $U$. (Some elementary properties of the space $\BV(U)$ are recalled in Section~\ref{sec:basicproperties} below; see also \cite{afpbook}.) If $u\in (L^2(U)\cap\BV(U))^d$, then we can define
\begin{equation}
J_{\mu,\lambda,\infty}(u)\coloneqq\int|u(x)-x|^2\,\dif\mu(x)+c\lambda\int\rho(x)^2\,\dif|Du|(x),\label{eq:Jinftydef}
\end{equation}
where 
\begin{equation}
c\coloneqq\int_{\mathbf{R}^{d}}\e^{-|y|}|y\cdot\mathrm{e}_1|\,\dif y.\label{eq:cdef}
\end{equation}
We will see in \cref{prop:gradfnal-unique-minimizer} below that $J_{\mu,\lambda,\infty}$ admits a unique minimizer $u_{\mu,\lambda,\infty}\in (L^2(U)\cap\BV(U))^d$. In \Cref{thm:convgamma}, we will then show in a quantitative sense that, if $U$ is sufficiently regular and the density $\rho$ is Lipschitz, then $u_{\mu,\la,\ga}$ converges to $u_{\mu,\la,\infty}$ as $\ga$ tends to infinity. The essential strategy here is to compare the functionals  $J_{\mu,\lambda,\infty}$ and $J_{\mu,\lambda,\gamma}$ and use their uniform convexity. An important technical complication is that $J_{\mu,\lambda,\infty}(u)$ is only defined for functions $u$ of bounded variation on $\overline{U}$ while the minimizer of $J_{\mu,\lambda,\gamma}$ may not (a priori) be of bounded variation. Therefore, to compare the functionals, we must first smooth their argument $u$ in a way that respects derivatives. Convolution by a smooth function works, but we first must dilate $u$ slightly since it is only defined on $\overline{U}$, not all of~$\R^d$. Moreover, this modification of the optimizer for $J_{\mu,\lambda,\gamma}$ needs to be performed in such a way that the functional does not increase too much.
It is this constraint that leads us to the requirement that the domains be effectively star-shaped (or that the more general condition in \cref{rem:starshapedthing} holds). %

The utility of the gradient functional \cref{eq:Jinftydef} in the proof of \cref{thm:maintheorem}
is apparent in \cref{prop:limitinggivescentroids} below. This proposition states that when $\lambda$ is large enough, the minimizer of the gradient functional recovers the centroids of the connected components
of the support of the measure $\mu$. The critical $\lambda$ is identified in terms of the $L^\infty$ norm of the solution to a PDE arising from the first-order conditions for the minimizer. We expect that further information about the behavior of the limiting functional could be obtained by further studying this PDE.

As mentioned, the gradient clustering functional \cref{eq:Jinftydef} only makes sense
for smooth measures. In order to show the convergence of the minimizers
of the weighted clustering functionals \cref{eq:Jdef} on empirical
distributions, we need to relate the minimizers of the finite-$\gamma$
problem for empirical distributions to the minimizers of the finite-$\gamma$ problem
for smooth distributions. We do this by proving a stability result
with respect to the $\infty$-Wasserstein metric $\mathcal{W}_{\infty}$,
which is \cref{prop:Winftystability} below. This works in combination
with a quantitative Glivenko--Cantelli-type result for the $\infty$-Wasserstein
metric proved in \cite{GS15}, and recalled in \cref{prop:glivenkocantelliwinfty} below.
However, since the latter result only holds for connected domains, we also need to
truncate the exponential weight in \cref{eq:Jdef}, which is done in \cref{sec:truncation}.

\subsection{Outline of the paper}
In \cref{sec:basicproperties} we establish some basic properties of $J_{\mu,\lambda,\gamma}$ and $J_{\mu,\lambda,\infty}$. In \cref{sec:Winftystab} we prove a stability result for $u_{\tilde{\mu},\lambda,\gamma}$ as $\tilde\mu\to\mu$ in the $\infty$-Wasserstein distance. In \cref{sec:convasgammatoinfty} we prove the convergence result for $u_{\mu,\lambda,\gamma}$ as $\gamma\to\infty$. In \cref{sec:limitingfnal} we show that the limiting functional $u_{{\mu},\lambda,\infty}$ recovers the centroids of the connected components of $\supp\mu$ as long as $\lambda$ is large enough. In \cref{sec:truncation} we prove a stability result when the exponential weight is truncated. In \cref{sec:mainthmproof} we put everything together to prove \cref{thm:maintheorem}.

\subsection*{Acknowledgments}
AD was partially supported by the NSF Mathematical Sciences Postdoctoral Fellowship program under grant no.\ DMS-2002118. JCM was partially supported by NSF grant DMS-1954357.

\section{Basic properties of the functionals}\label{sec:basicproperties}

As mentioned above, for a bounded open set $U \subset \Rd$, we denote by $\BV(U)$ the space of functions of bounded variation on $U$. This is the set of all functions $u \in L^1(U)$ whose derivatives are Radon measures. For $u \in \BV(U)$, we denote by $Du$ the gradient of $u$, which is thus a vector-valued Radon measure, and we denote by $|Du|$ its total variation. In particular, for every open set $V \subset U$, we have by \cite[Proposition~1.47]{afpbook} that
\begin{equation}  
\label{e.total.var}
|Du|(V) = \sup_\phi \int_V \phi \cdot \, \d D u  = \sup_\phi \sum_{i = 1}^d \int_V \phi_i \, \d D_i u,
\end{equation}
where the supremum is over all $\phi \in (\mathcal{C}_\mathrm{c}(V))^d$ (the space of $\R^d$-valued continuous functions supported on compact subsets of $V$) such that $\|\phi\|_{L^\infty(V)} \le 1$, with the understanding that 
\begin{equation*}  %
\|\phi\|_{L^\infty(V)} = \| \, |\phi| \, \|_{L^\infty(V)} = \esssup_{x \in V} \Ll( \sum_{i = 1}^d \phi_i^2(x)  \Rr) ^\frac 1 2.
\end{equation*}
When $u \in (\BV(U))^d$, the gradient $Du$ is a Radon measure taking values in the space of $d\times d$ matrices. Identifying such a matrix with an element of $\R^{d^2}$, %
we can still define the total variation measure $|Du|$ as above. (Thus, if $Du$ is in fact an $\R^{d\times d}$-valued function, then $|Du|(x)$ is the Frobenius norm of the matrix $Du(x)$.) We refer to \cite{afpbook} for a thorough exposition of the properties of $\BV$ functions.

In the remainder of this section, we collect some basic properties of the functionals
$J_{\mu,\lambda,\gamma}$. It is straightforward to see that, for
any $\gamma\in(0,\infty)$, the functional $J_{\mu,\lambda,\gamma}$ is uniformly
convex on $(L^2(\mu))^d$. Indeed, for every $u,v \in (L^2(\mu))^d$, we have
\begin{equation}
\frac12\left(J_{\mu,\lambda,\gamma}(u+v)+J_{\mu,\lambda,\gamma}(u-v)\right)-J_{\mu,\lambda,\gamma}(u)\ge\int v^2\,\dif\mu.
\label{eq:uniformconvexity}
\end{equation}
Since the functional is also coercive, the existence and uniqueness of the minimizer $u_{\mu,\lambda,\ga}$ follow. The next proposition covers the case when $\ga = \infty$. 
\begin{prop}
\label{prop:gradfnal-unique-minimizer}Let $U$ be a bounded open subset
of $\mathbf{R}^{d}$ and suppose that $\supp\mu=\overline{U}$. Suppose
furthermore that $\mu$ is absolutely continuous with respect to the
Lebesgue measure on $\overline{U}$ with a density $\rho\in\mathcal{C}(\overline{U})$ that is bounded away from zero on $\bar U$.
Then for any $\lambda>0$, the functional $J_{\mu,\lambda,\infty}$
admits a unique minimizer $u_{\mu,\lambda,\infty}\in (L^2(U)\cap\BV(U))^d$.
\end{prop}

\begin{proof}
We start by observing that the convexity property \cref{eq:uniformconvexity} is still valid for $\gamma = \infty$, for every $u,v \in (L^2(U) \cap \BV(U))^d$. 
Let $(u_k)_k$ be a sequence of functions in $(L^2(U)\cap\BV(U))^d$
such that 
\begin{equation}
\lim_{k\to\infty}J_{\mu,\lambda,\infty}(u_k)=\inf_{u \in (L^2(U)\cap\BV(U))^d} J_{\mu,\lambda,\infty}(u).\label{eq:convtoinf}
\end{equation}
Since $\rho$ is bounded away from zero, the functional $J_{\mu,\la,\infty}$ is coercive on $(L^2(U) \cap \BV(U))^d$.
By the Banach--Alaoglu theorem and \cite[Theorem~3.23]{afpbook} (the latter saying that sets $S$ of functions in $\BV(U)$ for which $\sup_{u\in S}\int_U|u|\,\dif x+ |Du|(U)<\infty$ are weakly-$*$ precompact), by
passing to a subsequence we can assume that there is a $u\in (L^2(U)\cap\BV(U))^d$
such that $u_k\to u$ weakly in $(L^2(U))^d$ and weakly-$*$ in $(\BV(U))^d$.
From the weak convergence in $(L^2(U))^d$ we see that
\[
\int|u(x)-x|^2\,\dif\mu(x)\le\liminf_{k\to\infty}\int|u_k(x)-x|^2\,\dif\mu(x).
\]
From the weak-$*$ convergence in $(\BV(U))^d$ we see that
\begin{align*}
\int_U\rho(x)^2\,\dif|Du|(x) 
& = \sup_{\phi} \int_U \rho(x)^2 \phi(x) \cdot \dif D u(x) \\
& \le\liminf_{k\to\infty} \sup_{\phi} \int_U\rho(x)^2\phi(x) \cdot \dif Du_k(x) \\
& =\liminf_{k\to\infty} \int_U\rho(x)^2\,\dif|Du_k|(x),
\end{align*}
where the supremum is over all $\phi \in (\mcl C_c(U))^{d^2}$ such that $\|\phi\|_{L^\infty(U)} \le 1$.
The last two displays and \cref{eq:convtoinf} imply that $J_{\mu,\lambda,\infty}(u)=\inf J_{\mu,\lambda,\infty}$,
so we can take $u_{\mu,\lambda,\infty}=u$. The uniqueness of~$u_{\mu,\lambda,\infty}$
follows from the uniform convexity \cref{eq:uniformconvexity}.
\end{proof}
A direct consequence of the convexity property \eqref{eq:uniformconvexity} is that, for every $\gamma \in (0,\infty)$ and $u\in (L^2(\mu))^d$, we have
\begin{align}
\int|u-u_{\mu,\lambda,\gamma}|^2\,\dif\mu & \le2\left(J_{\mu,\lambda,\gamma}(u)+J_{\mu,\lambda,\gamma}(u_{\mu,\lambda,\gamma})\right)-4J_{\mu,\lambda,\gamma}\left(\frac{u_{\mu,\lambda,\gamma}+u}{2}\right)\nonumber \\
 & \le2\left(J_{\mu,\lambda,\gamma}(u)-\inf J_{\mu,\lambda,\gamma}\right).\label{eq:applyuniformconvexity}
\end{align}
Under the assumptions of \Cref{prop:gradfnal-unique-minimizer}, the inequalities in \cref{eq:applyuniformconvexity} remain valid with $\gamma = \infty$, provided that we also impose that $u \in (L^2(U) \cap \BV(U))^d$. 
Another important fact will be that, for every $\gamma\in(0,\infty]$,
\begin{equation}
0 \le \inf J_{\mu,\lambda,\gamma}\le J_{\mu,\lambda,\gamma}(\centroid{\mu}(\mathbf{R}^{d}))=\int|x-\centroid{\mu}(\mathbf{R}^{d})|^2\,\dif\mu(x),
\label{eq:infJub}
\end{equation}
where we note that the right-hand side is the variance of a random variable
distributed according to $\mu$, and in particular is independent
of $\lambda$ and $\gamma$.

\section{Stability with respect to \texorpdfstring{$\infty$}{∞}-Wasserstein perturbations of the measure}\label{sec:Winftystab}
Throughout the paper, for any two measures $\mu$ and $\nu$ on $\Rd$, we let $\mathcal{W}_{\infty}(\mu,\nu)$
be the $\infty$-Wasserstein distance between $\mu$ and $\nu$, namely
\[
\mathcal{W}_{\infty}(\mu,\nu)=\adjustlimits\inf_{\pi}\esssup_{(x,y)\sim\pi}|x-y|,
\]
where the infimum is taken over all couplings $\pi$ of $\mu$ and
$\nu$. It is classical to verify that this infimum is achieved (see e.g. \cite[Proposition~2.1]{CDPJ08}). We call any $\pi$ achieving this infimum an \emph{$\infty$-optimal transport plan} from $\mu$ to $\nu$. In this section we prove that, for finite~$\gamma$, the minimizer $u_{\mu,\lambda,\gamma}$ is stable under $\infty$-Wasserstein perturbations of $\mu$.
\begin{prop}
\label{prop:Winftystability}
There is a universal constant $C$ such that the following holds.
Let $\gamma,\lambda,M\in (0,\infty)$ and
let $\mu,\tilde{\mu}$ be two probability measures on $\mathbf{R}^{d}$
with supports contained in a common Euclidean ball $B$ of diameter
$M$. There exists an $\infty$-optimal transport plan $\pi$ from
$\mu$ to $\tilde{\mu}$ such that
\begin{equation}
\int|u_{\mu,\lambda,\gamma}(x)-u_{\tilde{\mu},\lambda,\gamma}(\tilde{x})|^2\,\dif\pi(x,\tilde{x})\le C(M+1)^2(\gamma+1)\mathcal{W}_{\infty}(\mu,\tilde{\mu}).\label{eq:uutildecoupling}
\end{equation}
\end{prop}

\begin{proof}
Throughout the proof, $\lambda$ and $\gamma$ will remain fixed,
so we write $J_{\mu}=J_{\mu,\lambda,\gamma}$ and $u_{\mu}=u_{\mu,\lambda,\gamma}$.
(Nonetheless, we emphasize that the constant $C$ in the statement
of the theorem does \emph{not} depend on $\lambda$ or $\gamma$.)
Let $\pi$ be an $\infty$-optimal transport plan from $\mu$ to $\tilde{\mu}$.
We write the disintegration
\[
\dif\pi(x,\tilde{x})=\dif\nu(\tilde{x}\mid x)\,\dif\mu(x)
\]
and define
\[
\overline{u}(x)\coloneqq\int u_{\tilde{\mu}}(\tilde{x})\,\dif\nu(\tilde{x}\mid x).
\]
We have
\begin{align}
\inf J_{\tilde{\mu}} & =\int|u_{\tilde{\mu}}(\tilde{x})-\tilde{x}|^2\,\dif\tilde{\mu}(\tilde{x})+\lambda\gamma^{d+1}\iint\e^{-\gamma|\tilde{x}-\tilde{y}|}|u_{\tilde{\mu}}(\tilde{x})-u_{\tilde{\mu}}(\tilde{y})|\,\dif\tilde{\mu}(\tilde{x})\,\dif\tilde{\mu}(\tilde{y})\nonumber \\
 & =\iint|u_{\tilde{\mu}}(\tilde{x})-\tilde{x}|^2\,\dif\nu(\tilde{x}\mid x)\,\dif\mu(x)\nonumber \\
 & \qquad+\lambda\gamma^{d+1}\iiiint\e^{-\gamma|\tilde{x}-\tilde{y}|}|u_{\tilde{\mu}}(\tilde{x})-u_{\tilde{\mu}}(\tilde{y})|\,\dif\nu(\tilde{x}\mid x)\,\dif\mu(x)\,\dif\nu(\tilde{y}\mid y)\,\dif\mu(y).\label{eq:infJ}
\end{align}
For the first term on the right side of \cref{eq:infJ}, we write
\begin{align}
|u_{\tilde{\mu}}(\tilde{x})-\tilde{x}|^2 & =|u_{\tilde{\mu}}(\tilde{x})-x|^2-|x-\tilde{x}|^2 + 2(u_{\tilde{\mu}}(\tilde{x})-\tilde{x})\cdot(x-\tilde{x})\nonumber\\
&\ge|u_{\tilde{\mu}}(\tilde{x})-x|^2-3M|x-\tilde{x}|.\label{eq:firsttermthing}
\end{align}
For the second term on the right side of \cref{eq:infJ}, we note that, for $\mu$-a.e. $x,y$, on the
support of $\nu(\tilde{x}\mid x)\otimes\nu(\tilde{y}\mid y)$ we have,
writing $W:=\mathcal{W}_{\infty}(\mu,\tilde{\mu})$,
\[
|\tilde{y}-\tilde{x}|\le2W+|y-x|,
\]
so
\[
\e^{-\gamma|\tilde{x}-\tilde{y}|}\ge\e^{-2\gamma W}\e^{-\gamma|y-x|}.
\]
Thus we can write
\begin{align}
 & \iiiint\e^{-\gamma|\tilde{x}-\tilde{y}|}|u_{\tilde{\mu}}(\tilde{x})-u_{\tilde{\mu}}(\tilde{y})|\,\dif\nu(\tilde{x}\mid x)\,\dif\mu(x)\,\dif\nu(\tilde{y}\mid y)\,\dif\mu(y)\nonumber \\
 & \quad\ge\e^{-2\gamma W}\iint\e^{-\gamma|x-y|}\left(\iint|u_{\tilde{\mu}}(\tilde{x})-u_{\tilde{\mu}}(\tilde{y})|\,\dif\nu(\tilde{x}\mid x)\,\dif\nu(\tilde{y}\mid y)\right)\,\dif\mu(x)\,\dif\mu(y)\nonumber \\
 & \quad\ge\e^{-2\gamma W}\iint\e^{-\gamma|x-y|}|\overline{u}(x)-\overline{u}(y)|\,\dif\mu(x)\,\dif\mu(y),\label{eq:secondterm}
\end{align}
where we used Jensen's inequality in the last step. Substituting \cref{eq:firsttermthing}
and \cref{eq:secondterm} into \cref{eq:infJ}, we obtain
\begin{align*}
\inf J_{\tilde{\mu}} & \ge\iint|u_{\tilde{\mu}}(\tilde{x})-x|^2\,\dif\nu(\tilde{x}\mid x)\,\dif\mu(x)-3M\iint|x-\tilde{x}|\,\dif\pi(x,\tilde{x})\\
 & \qquad+\lambda\gamma^{d+1}\e^{-2\gamma W}\iint\e^{-\gamma|x-y|}|\overline{u}(x)-\overline{u}(y)|\,\dif\mu(x)\,\dif\mu(y)\\
 & \ge\int|\overline{u}(x)-x|^2\,\dif\mu(x)+\lambda\gamma^{d+1}\e^{-2\gamma W}\iint\e^{-\gamma|x-y|}|\overline{u}(x)-\overline{u}(y)|\,\dif\mu(x)\,\dif\mu(y)-3MW\\
 & \ge\e^{-2\gamma W}J_{\mu}(\overline{u})-3MW,
\end{align*}
where in the second step we again used Jensen's inequality.
Therefore, we have
\begin{equation}
\inf J_{\mu}\le J_{\mu}(\overline{u})\le\e^{2\gamma W}\left(\inf J_{\tilde{\mu}}+3MW\right)\le\inf J_{\tilde{\mu}}+3M\e^{2\gamma W}W+\left(\e^{2\gamma W}-1\right)M^2,\label{eq:infJchain}
\end{equation}
with the last inequality by \cref{eq:infJub}. By symmetry, this implies
that
\begin{equation}
\left|\inf J_{\tilde{\mu}}-\inf J_{\mu}\right|\le3M\e^{2\gamma W}W+(\e^{2\gamma W}-1)M^2.\label{eq:infdiff}
\end{equation}
Now we have, using the second and third inequalities of \cref{eq:infJchain},
as well as \cref{eq:applyuniformconvexity} and \cref{eq:infdiff}, that
\begin{align}
\int|\overline{u}-u_{\mu}|^2\,\dif\mu & \le2\left(J_{\mu}(\overline{u})-\inf J_{\mu}\right)\le2\left(\inf J_{\tilde{\mu}}-\inf J_{\mu}\right)+6M\e^{2\gamma W}W+2\left(\e^{2\gamma W}-1\right)M^2\nonumber \\
 & \le12M\e^{2\gamma W}W+4(\e^{2\gamma W}-1)M^2\le(M+1)^2Q((\gamma+1)\mathcal{W}_{\infty}(\mu,\tilde{\mu})),\label{eq:ubarminus}
\end{align}%
where we have defined $Q(t) \coloneqq 12\e^{2t}t+4(\e^{2t}-1)$.

The remainder of the proof is very similar to the second half of the
proof of \cite[Proposition~5.3]{DM21}. For each $\eps>0$, let $\mu_{\eps}$
be a measure on the ball $B$, absolutely continuous with respect
to the Lebesgue measure, and such that
\begin{equation}
\mathcal{W}_{\infty}(\mu,\mu_{\eps})\le\eps.\label{eq:muapproxmueps}
\end{equation}
Since $\mu_\eps$ is absolutely continuous with respect to the Lebesgue measure, by \cite[Theorems~5.5 and~3.2]{CDPJ08} there are maps $T_\eps$ and $\tilde{T}_\eps$ from $\supp\mu_\eps$ to $\supp\mu$ and $\supp\tilde\mu$, respectively, such that $(\id\times T_\eps)_*(\mu_\eps)$ is an $\infty$-optimal transport plan between $\mu_\eps$ and $\mu$ and similarly $(\id\times \tilde T_\eps)_*(\mu_\eps)$ 
is an $\infty$-optimal transport plan between $\mu_\eps$ and $\tilde \mu$.
We have
\begin{equation}
\begin{aligned}\int & |u_{\mu}(T_{\eps}(x))-u_{\tilde{\mu}}(\tilde{T}_{\eps}(x))|^2\,\dif\mu_{\eps}(x)\\
 & \le 2\int|u_{\mu}(T_{\eps}(x))-u_{\mu_{\eps}}(x)|^2\,\dif\mu_{\eps}(x)+2\int|u_{\mu_{\eps}}(x)-u_{\tilde{\mu}}(\tilde{T}_{\eps}(x))|^2\,\dif\mu_{\eps}(x).
\end{aligned}
\label{eq:splituTuTtildesquared}
\end{equation}
For the first term on the right side, we use \cref{eq:ubarminus} above
with $\mu\leftarrow\mu_{\eps}$ and $\tilde{\mu}\leftarrow\mu$ (so
that $\overline{u}\leftarrow u_{\mu}\circ T_{\eps}$):
\begin{align*}
\int & |u_{\mu}(T_{\eps}(x))-u_{\mu_{\eps}}(x)|^2\,\dif\mu_{\eps}(x)\le(M+1)^2Q({(\gamma+1)}\eps).
\end{align*}
For the second term on the right side, we use \cref{eq:ubarminus} above
with $\mu\leftarrow\mu_{\eps}$ and $\tilde{\mu}\leftarrow\tilde{\mu}$
(so that $\overline{u}\leftarrow u_{\tilde{\mu}}\circ\tilde{T}_{\eps}$):
\[
\int|u_{\mu_{\eps}}(x)-u_{\tilde{\mu}}(\tilde{T}_{\eps}(x))|^2\,\dif\mu_{\eps}(x)\le(M+1)^2Q({(\gamma+1)}\mathcal{W}_{\infty}(\mu_{\eps},\tilde{\mu})).
\]
Using the last two displays in \cref{eq:splituTuTtildesquared}, we
get
\begin{align}
\int&|u_{\mu}(T_{\eps}(x))-u_{\tilde{\mu}}(\tilde{T}_{\eps}(x))|^2\,\dif\mu_{\eps}(x)\nonumber\\&\le2(M+1)^2Q({(\gamma+1)}\eps)+2(M+1)^2Q({(\gamma+1)}\mathcal{W}_{\infty}(\mu_{\eps},\tilde{\mu})).\label{eq:uTuTtilde}
\end{align}
We can find a sequence $\eps_k\downarrow0$ and a coupling $\pi$
of $\mu$ and $\tilde{\mu}$ such that $(T_{\eps_k},\tilde{T}_{\eps_k})_{*}\mu_{\eps_k}\to\pi$
as $k\to\infty$. Taking $\eps=\eps_k$ in \cref{eq:uTuTtilde}, and
then taking the limit as $k\to\infty$,
we get 
\begin{equation}
\int|u_{\mu,\lambda,\gamma}(x)-u_{\tilde{\mu},\lambda,\gamma}(\tilde{x})|^2\,\dif\pi(x,\tilde{x})\le 2(M+1)^2Q({(\gamma+1)}\mathcal{W}_{\infty}(\mu,\tilde{\mu})).
\label{eq:boundwithQ}\end{equation}
Hence, since, $Q$ is smooth, $Q(0)=0$, and the left side of \cref{eq:boundwithQ} is also evidently bounded above by $M^2$, we obtain the desired inequality
\cref{eq:uutildecoupling}.

It remains to show that $\pi$ is an $\infty$-optimal transport plan.
This follows by using \cref{eq:muapproxmueps} to note that
\begin{align*}
\esssup_{x\sim\mu_{\eps}}|T_{\eps}(x)-\tilde{T}_{\eps}(x)| & \le\esssup_{x\sim\mu_{\eps}}|T_{\eps}(x)-x|+\esssup_{x\sim\mu_\eps}|x-\tilde{T}_{\eps}(x)|\le\eps+\mathcal{W}_{\infty}(\mu_{\eps},\tilde{\mu}),
\end{align*}
and then taking limits along the subsequence $\ep_k  \downarrow 0$.
\end{proof}

\section{Convergence as \texorpdfstring{$\gamma\to\infty$}{γ→∞}}\label{sec:convasgammatoinfty}
In this section we show that, under suitable assumptions on $U$ and $\mu$, the optimizer~$u_{\mu,\lambda,\gamma}$ converges to $u_{\mu,\lambda,\infty}$ as $\gamma\to\infty$. In essence, we will obtain this by showing a quantitative version of the fact that the functional $J_{\mu,\la,\ga}$ $\Gamma$-converges to $J_{\mu,\la,\infty}$ as $\gamma$ tends to infinity.
\begin{thm}
\label{thm:convgamma}Assume that $U=\supp\mu$ is effectively star-shaped
and has a Lipschitz boundary, and that the measure $\mu$ has a density
with respect to the Lebesgue measure that is Lipschitz
on $U$ and is bounded away from zero. Then there exists a constant
$C<\infty$ such that, for every $\lambda\in (0,\infty)$, we have
\begin{equation}
|\inf J_{\mu,\lambda,\infty}-\inf J_{\mu,\lambda,\gamma}|+\int|u_{\mu,\lambda,\infty}-u_{\mu,\lambda,\gamma}|^2\,\dif\mu\le C\gamma^{-1/3}.\label{eq:gammaconvconcl}
\end{equation}
\end{thm}

\begin{proof}
Without loss of generality, assume that the point
$x_{*}$ in \cref{def:effectivelystarshaped} is the origin, and that
the constant $C_{*}$ appearing there is $1$. We denote by $\rho$ the density of $\mu$ with respect to the Lebesgue measure.
By \cite[Theorem~5.4.1]{evansbook}, 
we can and do extend $\rho$ to a Lipschitz function
on $\mathbf{R}^{d}$, which we can also prescribe to vanish outside
of a bounded set. 
Throughout the proof, we will leave $\mu,\lambda$
fixed, and write $u_{\gamma}=u_{\mu,\lambda,\ga}$ and $J_{\gamma}=J_{\mu,\lambda,\ga}$.
The constant $C$ may depend on $\mu$
but not on $\gamma$ or $\lambda$, and may change over the course
of the argument. We let $U_{\eps}$ be the $\eps$-enlargement of $U$
as in \cref{def:effectivelystarshaped}.

For every $\eps\in(0,1)$, $\gamma\in(0,\infty]$, and $x\in U_{\eps}$,
we define
\[
\tilde{u}_{\gamma,\eps}(x)\coloneqq u_{\gamma}((1-\eps)x),
\]
and for every $x\in U$, we define 
\[
u_{\gamma,\eps}(x)\coloneqq(\tilde{u}_{\gamma,\eps}*\chi_{\eps})(x),
\]
where $*$ denotes the convolution operator, $\chi\in\mathcal{C}_{\mathrm{c}}^{\infty}(\mathbf{R}^{d};\mathbf{R}_{+})$
is a nonnegative smooth function with compact support in the unit
ball satisfying
\begin{equation}
\int_{\mathbf{R}^{d}}\chi(x)\,\dif x=1\qquad\text{and}\qquad\int_{\mathbf{R}^{d}}x\chi(x)\,\dif x=0,\label{eq:chinormalize}
\end{equation}
and where we have set $\chi_{\eps}:=\eps^{-d}\chi(\eps^{-1}\cdot)$.

\emph{Step 1}. We show that, for every $\gamma\in(0,\infty)$,
\begin{equation}
\begin{aligned}\int_{U_{\eps}}|\tilde{u}_{\gamma,\eps}(x)-x|^2\rho(x)\,\dif x & +\lambda\gamma^{d+1}\iint_{U_{\eps}^2}\e^{-\gamma|x-y|}|\tilde{u}_{\gamma,\eps}(x)-\tilde{u}_{\gamma,\eps}(y)|\rho(x)\rho(y)\,\dif x\,\dif y\\
 & \le J_{\gamma}(u_{\gamma})+C\eps.
\end{aligned}
\label{eq:gotoexpanded}
\end{equation}
To prove this, we bound the first term on the left side of \cref{eq:gotoexpanded} by 
\begin{align*}
\int_{U_\eps}|\tilde{u}_{\gamma,\eps}(x)-x|^2\rho(x)\,\dif x & \le(1-\eps)^{-d}\int_{U}\left|u_{\gamma}(x)-\frac{x}{1-\eps}\right|^2\rho\left(\frac{x}{1-\eps}\right)\,\dif x\\
 & \le\int_{U}|u_{\gamma}(x)-x|^2\rho(x)\,\dif x+C\eps,
\end{align*}
where in the second inequality we used the fact that $\rho$ is Lipschitz.
For the second term on the left side of \cref{eq:gotoexpanded}, we proceed similarly, noting that
\begin{align*}
\gamma^{d+1}\iint_{U_{\eps}^2} & \e^{-\gamma|x-y|}|\tilde{u}_{\gamma,\eps}(x)-\tilde{u}_{\gamma,\eps}(y)|\,\dif\mu(x)\,\dif\mu(y)\\
 & \le\frac{\gamma^{d+1}}{(1-\eps)^{2{d}}}\iint_{U^2}\e^{-\gamma|x-y|/(1-\eps)}|u_{\gamma}(x)-u_{\gamma}(y)|\rho\left(\frac{x}{1-\eps}\right)\rho\left(\frac{y}{1-\eps}\right)\,\dif x\,\dif y\\
 & \le\frac{\gamma^{d+1}}{(1-\eps)^{2{d}}}\iint_{U^2}\e^{-\gamma|x-y|}|u_{\gamma}(x)-u_{\gamma}(y)|\rho\left(\frac{x}{1-\eps}\right)\rho\left(\frac{y}{1-\eps}\right)\,\dif x\,\dif y\\
 & \le\frac{\gamma^{d+1}}{(1-\eps)^{2{d}}}\iint_{U^2}\e^{-\gamma|x-y|}|u_{\gamma}(x)-u_{\gamma}(y)|\rho(x)\rho(y)\,\dif x\,\dif y+C\eps.
\end{align*}
It is in this calculation that the star-shaped property is crucial:
in the second inequality, we used that the map sending $U_{\eps}$ to $U$ (i.e. the map $x\mapsto x/(1-\eps)$)
is contractive. We also used \cref{eq:infJub} and again the fact that $\rho$ is Lipschitz. Combining the last two displays, we obtain \cref{eq:gotoexpanded}.

\emph{Step 2. }We show that, for every $\gamma\in(0,\infty)$,
\begin{equation}
J_{\gamma}(u_{\gamma,\eps})\le J_{\gamma}(u_{\gamma})+C\eps.\label{eq:ugammaepsok}
\end{equation}
Using \cref{eq:chinormalize}, we can write
\begin{align*}
\int_{U}|u_{\gamma,\eps}(x)-x|^2\,\dif\mu(x) & =\int_{U}\left|\int_{U_{\eps}}(\tilde{u}_{\gamma,\eps}(y)-y)\chi_{\eps}(x-y)\,\dif y\right|^2\rho(x)\,\dif x\\
 & \le\int_{U_{\eps}}|\tilde{u}_{\gamma,\eps}(y)-y|^2\int_{\mathbf{R}^{d}}\chi_{\eps}(x-y)\rho(x)\,\dif x\,\dif y.
\end{align*}
Since $\rho$ is Lipschitz, the inner integral is close to
$\rho(y)$, up to an error bounded by $C\eps$, and we thus get that
\begin{equation}
\int_{U}|u_{\gamma,\eps}(x)-x|^2\,\dif\mu(x)\le\int_{U_{\eps}}|\tilde{u}_{\gamma,\eps}(x)-x|^2\rho(x)\,\dif x+C\eps.\label{eq:convol.l2}
\end{equation}
We also have
\begin{align*}
\gamma^{d+1}\iint_{U^2} & \e^{-\gamma|x-y|}|u_{\gamma,\eps}(x)-u_{\gamma,\eps}(y)|\rho(x)\rho(y)\,\dif x\,\dif y\\
 & \le\gamma^{d+1}\iint_{U^2}\e^{-\gamma|x-y|}\left|\int_{\mathbf{R}^{d}}[\tilde{u}_{\gamma,\eps}(x-z)-\tilde{u}_{\gamma,\eps}(y-z)]\chi_{\eps}(z)\,\dif z\right|\rho(x)\rho(y)\,\dif x\,\dif y\\
 & \le\gamma^{d+1}\iint_{U^2}\int_{\mathbf{R}^{d}}\e^{-\gamma|x-y|}\left|\tilde{u}_{\gamma,\eps}(x)-\tilde{u}_{\gamma,\eps}(y)\right|\chi_{\eps}(z)\rho(x+z)\rho(y+z)\,\dif z\,\dif x\,\dif y\\
 & \le\gamma^{d+1}\iint_{U_{\eps}^2}\e^{-\gamma|x-y|}\left|\tilde{u}_{\gamma,\eps}(x)-\tilde{u}_{\gamma,\eps}(y)\right|\left(\int_{\mathbf{R}^{d}}\chi_{\eps}(z)\rho(x+z)\rho(y+z)\,\dif z\right)\,\dif x\,\dif y\\
 & \le\gamma^{d+1}\iint_{U_{\eps}^2}\e^{-\gamma|x-y|}\left|\tilde{u}_{\gamma,\eps}(x)-\tilde{u}_{\gamma,\eps}(y)\right|\rho(x)\rho(y)\,\dif x\,\dif y+C\eps,
\end{align*}
where in the last step we used \cref{eq:gotoexpanded}, \cref{eq:infJub}, and the fact
that $\rho$ is Lipschitz. Combining the
last two displays with \cref{eq:gotoexpanded} yields \cref{eq:ugammaepsok}.

\emph{Step 3. }We show that, for every $\gamma\in[1,\infty)$ and
$\eps\in(0,1]$,
\begin{equation}
J_{\infty}(u_{\gamma,\eps})\le J_{\gamma}(u_{\gamma})+C\eps+\frac{C}{\gamma\eps^2}.\label{eq:Jinftyuga}
\end{equation}
In view of \eqref{eq:ugammaepsok}, it suffices to show \eqref{eq:Jinftyuga} with $J_\ga(u_\ga)$ replaced by $J_\ga(u_{\ga,\ep})$. 
We start by using the fact that $\|D^2u_{\gamma,\eps}\|_{L^{\infty}(\mu)}\le C\eps^{-2}$ to write
\begin{align}
\gamma^{d+1}\iint_{U^2} & \e^{-\gamma|x-y|}|u_{\gamma,\eps}(x)-u_{\gamma,\eps}(y)|\rho(x)\rho(y)\,\dif x\,\dif y\nonumber \\
 & \ge\gamma^{d+1}\iint_{U^2}\e^{-\gamma|x-y|}|Du_{\gamma,\eps}(x)\cdot(x-y)|\rho(x)\rho(y)\,\dif x\,\dif y\nonumber \\
 & \qquad-C\gamma^{d+1}\iint_{U^2}\e^{-\gamma|x-y|}\frac{|x-y|^2}{\eps^2}\rho(x)\rho(y)\,\dif x\,\dif y.\label{eq:taylor}
\end{align}
Since $\rho$ is bounded and
\begin{equation}
\gamma^{d+1}\int_{\mathbf{R}^{d}}\e^{-\gamma|x-y|}|x-y|^2\,\dif y=\gamma^{-1}\int_{\mathbf{R}^{d}}\e^{-|y|}|y|^2\,\dif y,\label{eq:chgvar}
\end{equation}
we see that the second integral on the right-hand side of \cref{eq:taylor}
is bounded by $C\gamma^{-1}\eps^{-2}$. Next, we aim to compare the
first integral on the right-hand side of \cref{eq:taylor} with the
same quantity with $\rho(y)$ replaced by $\rho(x)$. Since $\rho$
is Lipschitz and $\|D u_{\gamma,\eps}\|_{L^{\infty}(\mu)}\le C\eps^{-1}$,
the difference between these two quantities is bounded by
\[
C\eps^{-1}\gamma^{d+1}\iint_{U^2}\e^{-\gamma|x-y|}|x-y|^2\rho(x)\rho(y)\,\dif x\,\dif y\le C\gamma^{-1}\eps^{-1},
\]
using again \cref{eq:chgvar} and the boundedness of $\rho$. To complete this step, it remains to argue that 
\begin{equation}
\label{e.fix}
\gamma^{d+1}  \iint_{U^2}\e^{-\gamma|x-y|}|Du_{\gamma,\eps}(x)\cdot(x-y)|\rho(x)^2\,\dif x\,\dif y
\ge c \int \rho(x)^2 |D u_{\ga,\ep}(x)| \, \d x + C \ga^{-1} \ep^{-1}. 
\end{equation}
Recalling \cref{eq:cdef}, we see that the first term on the right-hand side above can be rewritten as
\begin{equation*}  %
\gamma^{d+1}  \int_{U} \int_{\Rd}\e^{-\gamma|x-y|}|Du_{\gamma,\eps}(x)\cdot(x-y)|\rho(x)^2\,\dif y\,\dif x.
\end{equation*}
For every $\de > 0$, we denote $U^\de := \{x \in U : \dist(x,\dr U) \le \de\}$.
Since $\|D u_{\ga,\ep}\|_{L^\infty(\mu)} \le C \ep^{-1}$, the inequality \eqref{e.fix} will follow from the fact that 
\begin{align}  
\label{e.fix2}
 \ga^{d+1} \int_U \int_{\Rd \setminus U} e^{-\ga|x-y|} |x-y| \, \d y \, \d x \le C \ga^{-1}. 
\end{align}
Since $U$ has a Lipschitz boundary, there exists $\de > 0$ such that for every $0 < \eta < \eta' < \delta$, the Lebesgue measure of $U^{\eta'} \setminus U^\eta$ is at most $C(\eta'-\eta)$. Therefore,
\begin{align*}
& \ga^{d+1} \int_U \int_{\Rd \setminus U} e^{-\ga|x-y|} |x-y| \, \d y \, \d x \\
& \qquad \le C \ga^{d+1} e^{-\de\ga} + \ga^{d+1} \sum_{k = 0}^{\lceil \de \ga \rceil} \int_{U^{(k+1)\ga^{-1}} \setminus U^{k\ga^{-1}}} \int_{\Rd \setminus U} e^{-\ga|x-y|} |x-y| \, \d y \, \d x
\\
& \qquad \le C \ga^{d+1} e^{-\de\ga}+ \ga^{d+1}  \sum_{k = 0}^{\lceil  \de \ga \rceil}  e^{-\frac{\ga k}{2}} \int_{U^{(k+1)\ga^{-1}} \setminus U^{k\ga^{-1}}} \int_{\Rd} e^{-\frac{\ga|x-y|}{2}} |x-y| \, \d y \, \d x
\\
&\qquad \le C \ga^{d+1} e^{-\de\ga} + C  \ga^{-1} \sum_{k = 0}^{\lceil \de \ga \rceil}  e^{-\frac{\ga k}{2}}  \\&\qquad\le C \ga^{-1}.
\end{align*}
This is \eqref{e.fix2}. Combining these estimates with \eqref{eq:ugammaepsok} yields \cref{eq:Jinftyuga}.

\emph{Step 4}. We show that
\begin{equation}
\int_{U_{\eps}}\left|\tilde{u}_{\infty,\eps}(x)-x\right|^2\rho(x)\,\dif x+c\lambda\iint_{U_{\eps}^2}\rho(x)^2\,\dif|D\tilde{u}_{\infty,\eps}|(x)\le J_{\infty}(u_{\infty})+C\eps.\label{eq:tduinf}
\end{equation}
This follows from the fact that the the left side of \cref{eq:tduinf}
can be rewritten as
\[
(1-\eps)^{{-d}}\int_{U}\left|u_{\infty}(x)-\frac{x}{1-\eps}\right|^2\rho\left(\frac{x}{1-\eps}\right)\,\dif x+\frac{c\lambda}{(1-\eps)^{{d+1}}}\iint_{U^2}\rho\left(\frac{x}{1-\eps}\right)^2\,\dif|Du_{\infty}|(x),
\]
and from the fact that $\rho$ is Lipschitz.

\emph{Step 5.} We show that
\begin{equation}
J_{\infty}(u_{\infty,\eps})\le J_{\infty}(u_{\infty})+C\eps.\label{eq:uinf}
\end{equation}
Arguing in the same way as for \cref{eq:convol.l2}, we see that
\begin{equation}
\int_{U}|u_{\infty,\eps}(x)-x|^2\,\dif\mu(x)\le\int_{U_{\eps}}|\tilde{u}_{\infty,\eps}(x)-x|^2\rho(x)\,\dif x+C\eps.\label{eq:re.convolution}
\end{equation}
For the second term, we notice that by \cite[Proposition 3.2]{afpbook},
we have
\[
D(\tilde{u}_{\infty,\eps}*\chi_{\eps})=D\tilde{u}_{\infty,\eps}*\chi_{\eps},
\]
and thus
\begin{align*}
\int_{U}\rho(x)^2|D(\tilde{u}_{\infty,\eps}*\chi_{\eps})|(x)\,\dif x & \le\int_{U}\int_{U_{\eps}}\rho(x)^2\chi_{\eps}(x-y)\,\dif|D\tilde{u}_{\infty,\eps}|(y)\,\dif x\\
 & \le\int_{U_{\eps}}\rho(y)^2\,\dif|D\tilde{u}_{\infty,\eps}|(y)+C\eps,
\end{align*}
where we used \cref{eq:tduinf}, \cref{eq:infJub}, and the fact that
$\rho$ is Lipschitz in the last step.
Combining this with \cref{eq:re.convolution} and using \cref{eq:tduinf}
once more, we obtain \cref{eq:uinf}.

\emph{Step 6. }We show that
\begin{equation}
J_{\gamma}(u_{\infty,\eps})\le J_{\infty}(u_{\infty})+C\eps+\frac{C}{\gamma\eps^2}.\label{eq:Jgammauinf}
\end{equation}
We decompose the fusion term of $J_{\gamma}(u_{\infty,\eps})$ into
\begin{align}
\gamma^{d+1} & \iint_{U^2}\e^{-\gamma|x-y|}|u_{\infty,\eps}(x)-u_{\infty,\eps}(y)|\rho(x)\rho(y)\,\dif x\,\dif y\nonumber \\
 & \le\gamma^{d+1}\iint_{U^2}\e^{-\gamma|x-y|}|D u_{\infty,\eps}(x)\cdot(x-y)|\rho(x)\rho(y)\,\dif x\,\dif y\nonumber \\
 & \qquad+C\gamma^{d+1}\iint_{U^2}\e^{-\gamma|x-y|}\frac{|x-y|^2}{\eps^2}\rho(x)\rho(y)\,\dif x\,\dif y,\label{eq:decompose-again}
\end{align}
and estimate each of these integrals in turn. The second integral on the right side
is the same as the second integral in~\cref{eq:taylor}, and thus is
bounded by $C\gamma^{-1}\eps^{-2}$. We next aim to compare the first
integral on the right-hand side of \cref{eq:decompose-again} with the one
where $\rho(y)$ is replaced by $\rho(x)$. Since $\rho$ is Lipschitz, 
the difference between these two quantities is bounded by
\[
C\gamma^{d+1}\iint_{U^2}\e^{-\gamma|x-y|}|Du_{\infty,\eps}(x)||x-y|^2\,\dif x\,\dif y\le C\gamma^{-1}\int_{U}|Du_{\infty,\eps}(x)|\,\dif x\le C\gamma^{-1},
\]
where we used \cref{eq:uinf} and the fact that $\rho$ is bounded above
and below in the last step. Then it remains to estimate
\begin{align*}
\gamma^{d+1} & \iint_{U^2}\e^{-\gamma|x-y|}|D u_{\infty,\eps}(x)\cdot(x-y)|\rho(x)^2\,\dif x\,\dif y\\
 & \le\int_{\mathbf{R}^2}\e^{-|y|}|y\cdot\mathrm{e}_1|\,\dif y\int_{U}|D u_{\infty,\eps}(x)|\rho(x)^2\,\dif x=c\int_{U}|D u_{\infty,\eps}(x)|\rho(x)^2\,\dif x,
\end{align*}
where we recalled \cref{eq:cdef} in the last step. Thus we have
\[
J_{\gamma}(u_{\infty,\eps})\le J_{\infty}(u_{\infty,\eps})+C\gamma^{-1}\eps^{-2},
\]
and inequality \cref{eq:Jgammauinf} then follows using \cref{eq:uinf}.

\emph{Step 7}. We can now conclude the proof. We take $\eps\coloneqq\gamma^{-1/3}$,
and using \cref{eq:Jinftyuga} and \cref{eq:Jgammauinf}, we see that
\[
J_{\infty}(u_{\infty})\le J_{\infty}(u_{\gamma,\gamma^{-1/3}})\le J_{\gamma}(u_{\gamma})+C\gamma^{-1/3}\le J_{\gamma}(u_{\infty,\gamma^{-1/3}})+C\gamma^{-1/3}\le J_{\infty}(u_{\infty})+C\gamma^{-1/3}.
\]
From this, we deduce that
\begin{equation}
|J_{\infty}(u_{\infty})-J_{\gamma}(u_{\gamma})|\le C\gamma^{-1/3},\label{eq:Jsclose}
\end{equation}
and moreover that
\begin{equation}
0\le J_{\infty}(u_{\gamma,\gamma^{-1/3}})-J_{\infty}(u_{\infty})\le C\gamma^{-1/3}.\label{eq:Jinftyugamma}
\end{equation}
By \cref{eq:applyuniformconvexity} and \cref{eq:Jinftyugamma}, we obtain
\begin{equation}
\int|u_{\gamma,\gamma^{-1/3}}-u_{\infty}|^2\,\dif\mu\le C\gamma^{-1/3}.\label{eq:uapproxclosetouinf}
\end{equation}
Using \cref{eq:applyuniformconvexity} and \cref{eq:ugammaepsok}, we
also infer that
\begin{equation}
\int|u_{\gamma,\gamma^{-1/3}}-u_{\gamma}|^2\,\dif\mu\le C\gamma^{-1/3}.\label{eq:uapproxclosetougamma}
\end{equation}
Combining \cref{eq:Jsclose}, \cref{eq:uapproxclosetouinf}, and \cref{eq:uapproxclosetougamma}
yields \cref{eq:gammaconvconcl}.
\end{proof}

\begin{rem}\label{rem:starshapedthing}
In the proof of \Cref{thm:convgamma}, the assumption that $U$ is effectively star-shaped could be replaced by the following weaker assumption: that there exist $L < \infty$ and, for every $\eps > 0$ sufficiently small, a $1$-Lipschitz injective map $P_\ep : U_\ep \to U$ with $L$-Lipschitz inverse. In \Cref{thm:maintheorem}, we could then assume that the same property holds for each of the sets $U_1, \ldots, U_L$ in place of the assumption that these sets are effectively star-shaped.
\end{rem}

\section{Properties of the limiting functional}\label{sec:limitingfnal}

In this section we show that if $\lambda$ is large enough, then the minimizer $u_{\mu,\lambda,\infty}$ of $J_{\mu,\lambda,\infty}$ recovers the connected components of $\supp\mu$.

\begin{prop}
\label{prop:limitinggivescentroids}Let $\mu$ be a probability measure
on $\mathbf{R}^{d}$ satisfying the conditions of \cref{thm:maintheorem},
so its support is the disjoint union of $\overline{U_1}\sqcup\cdots\sqcup\overline{U_{L}}$.
There is a $\lambda_{\mathrm{c}}<\infty$ such that if $\lambda\ge\lambda_{\mathrm{c}}$,
then $u_{\mu,\lambda,\infty}(x)=\centroid{\mu}(U_\ell)$ for all
$x\in U_\ell$, $\ell\in\{1,\ldots,L\}$.
\end{prop}

\begin{proof}
Let $u(x)=\centroid{\mu}(U_\ell)$ for all $x\in U_\ell$, $\ell\in\{1,\ldots,L\}$. Since
the gradient of $u$ is zero on each $U_\ell$, we have
\[
J_{\mu,\lambda,\infty}(u)=\sum_{\ell=1}^L\int_{U_\ell}|u(x)-x|^2\,\dif\mu(x).
\]
Let $U=\bigcup_{\ell=1}^L U_\ell$, $p>d$, and let $W^{1,p}(U)$ denote the usual Sobolev space with regularity~$1$ and integrability $p$. Note that $W^{1,p}(U)$ embeds continuously into $\mathcal{C}(\bar U)$ by Morrey's inequality; see \cite[Theorem~4.12]{afbook}. Let $\psi\in (W^{1,p}(U))^{d\times d}$ be a weak solution to the PDE
\begin{align}
2\rho(x)(u(x)_j-x_j)-c\sum_{k=1}^{d}D_k(\rho^2\psi_{jk})(x) & =0,\qquad x\in U,j=1,\ldots,d;\label{eq:psidef}\\
\psi|_{\partial U} & \equiv0.\label{eq:psiBC}
\end{align}
We note that the problem \cref{eq:psidef}--\cref{eq:psiBC} separates into $dL$ problems, one for each $j$ and $\ell$. Each problem can be solved by \cite[Theorem~2.4]{BS90} (which follows the approach introduced in \cite{Bog79,Bog80}). 
We have, for every $v\in (L^2(U) \cap \BV(U))^d$, 
\begin{align*}
J&_{\mu,\lambda,\infty}(u+v) =\int_U|u(x)+v(x)-x|^2\,\dif\mu(x)+c\lambda\int_U\rho(x)^2\,\dif|Dv|(x)\\
 & =J_{\mu,\lambda,\infty}(u)+\int_U\left(2 (u(x)-x)\cdot v(x)+|v(x)|^2\right)\,\dif\mu(x)+ c\lambda\int_U\rho(x)^2\,\dif|Dv|(x).
\end{align*}
A minor variant of \cref{e.total.var} takes the form
\begin{equation*}  %
\int_U\rho(x)^2\,\dif|Dv|(x) = \sup \Ll\{ \int_U  \rho(x)^2 \phi(x) \cdot \d D v(x), \ \phi \in (\mcl C(\bar U))^{d\times d} \text{ s.t. } \|\phi\|_{L^\infty(U)} \le 1\Rr\}.
\end{equation*}
Selecting $\phi = \psi/\|\psi\|_{L^\infty(U)}$, and using the assumption that
$\lambda\ge\|\psi\|_{L^\infty(U)}$, we obtain
\begin{align*}
J_{\mu,\lambda,\infty}(u+v) & \ge J_{\mu,\lambda,\infty}(u)+\int\left(2(u(x)-x)\cdot v(x)+|v(x)|^2\right)\,\dif\mu(x)\\
 & \qquad+c\sum_{j,k=1}^{d}\int \rho(x)^2\psi_{jk}(x)D_kv_j(x)\,\dif x\\
 & =J_{\mu,\lambda,\infty}(u)+\int\left(2(u(x)-x)\cdot v(x)+|v(x)|^2\right)\,\dif\mu(x)\\
 & \qquad-\sum_{j=1}^{d}\int 2\rho(x)(u(x)_j-x_j)(x)v_j(x)\,\dif x\\
 & =J_{\mu,\lambda,\infty}(u)+\int|v(x)|^2\,\dif\mu(x)\\
 & \ge J_{\mu,\lambda,\infty}(u),
\end{align*}
where we used \cref{eq:psidef} for the first equality. This implies that $u_{\mu,\lambda,\infty}=u$,
and hence the statement of the proposition with $\lambda_{\mathrm{c}}=\|\psi\|_{L^\infty(U)}$.
\end{proof}

\section{Truncation}\label{sec:truncation}

In this section we prove a stability result for when we truncate the
exponential weight. For $\gamma,\omega\in(0,\infty)$, we define the truncated
functional
\begin{equation}
\begin{aligned}
\overline{J}&_{\mu,\lambda,\gamma,\omega}(u)\\&\coloneqq\int|u(x)-x|^2\,\dif\mu(x)+\lambda\gamma^{d+1}\iint\e^{-\gamma|x-y|}\mathbf{1}\{|x-y|\le\omega\}|u(x)-u(y)|\,\dif\mu(x)\,\dif\mu(y).\end{aligned}\label{eq:Jbardef}
\end{equation}
The functional $\overline{J}_{\mu,\lambda,\gamma,\omega}$ is uniformly
convex and satisfies \cref{eq:uniformconvexity} and \cref{eq:applyuniformconvexity}
in the same way as $J_{\mu,\lambda,\gamma}$. Let $\overline{u}_{\mu,\lambda,\gamma,\omega}$
be the (unique) minimizer of $\overline{J}_{\mu,\lambda,\gamma,\omega}$.
\begin{prop}
\label{prop:truncate}Let $\gamma,\lambda,\omega>0$ and let $\mu$ be
a probability measure on $\mathbf{R}^{d}$ with compact support. Let $M\coloneqq \diam\supp\mu$.
Then we have
\begin{equation}
\int\left|\overline{u}_{\mu,\lambda,\gamma,\omega}(x)-u_{\mu,\lambda,\gamma}(x)\right|^2\,\dif\mu(x)\le2M\lambda\gamma^{d+1}\e^{-\gamma\omega}.\label{eq:truncate}
\end{equation}
\end{prop}

In light of this statement, we 
define
\begin{equation}
\bar u_{\mu,\la,\ga} := \bar u_{\mu,\la,\ga,(d+4/3)\gamma^{-1}\log\gamma}. \label{eq:Jbarabbrvdef}
\end{equation}
Then \cref{eq:truncate} implies that
\begin{equation}
\int\left|\overline{u}_{\mu,\lambda,\gamma}(x)-u_{\mu,\lambda,\gamma}(x)\right|^2\,\dif\mu(x)\le2M\lambda\gamma^{-1/3}.\label{eq:Jbarabbrvtruncatebound}
\end{equation}

\begin{proof}[Proof of \cref{prop:truncate}]
Subtracting \cref{eq:Jdef} from \cref{eq:Jbardef}, we obtain
\[
\overline{J}_{\mu,\lambda,\gamma,\omega}(u)-J_{\mu,\lambda,\gamma}(u)=\lambda\gamma^{d+1}\iint\e^{-\gamma|x-y|}\mathbf{1}\{|x-y|>\omega\}|u(x)-u(y)|\,\dif\mu(x)\,\dif\mu(y).
\]
Taking $u=u_{\mu,\lambda,\gamma}$, we get
\begin{align*}
\overline{J}_{\mu,\lambda,\gamma,\omega}&(u_{\mu,\lambda,\gamma})-\inf J_{\mu,\lambda,\gamma} \\& =\lambda\gamma^{d+1}\iint\e^{-\gamma|x-y|}\mathbf{1}\{|x-y|>\omega\}|u_{\mu,\lambda,\gamma}(x)-u_{\mu,\lambda,\gamma}(y)|\,\dif\mu(x)\,\dif\mu(y)\\
 & \le M\lambda\gamma^{d+1}\e^{-\gamma\omega},
\end{align*}
and similarly,
\begin{align*}
J_{\mu,\lambda,\gamma} & (\overline{u}_{\mu,\lambda,\gamma,\omega})-\inf\overline{J}_{\mu,\lambda,\gamma,\omega}\\
 & =-\lambda\gamma^{d+1}\iint\e^{-\gamma|x-y|}\mathbf{1}\{|x-y|>\omega\}|\overline{u}_{\mu,\lambda,\gamma,\omega}(x)-\overline{u}_{\mu,\lambda,\gamma,\omega}(y)|\,\dif\mu(x)\,\dif\mu(y)\le 0.
\end{align*}
Therefore, using \cref{eq:applyuniformconvexity} and the last two displays
we have
\begin{align*}
\int & \left|\overline{u}_{\mu,\lambda,\gamma,\omega}(x)-u_{\mu,\lambda,\gamma}(x)\right|^2\,\dif\mu(x)\\
 & \le2\left(J_{\mu,\lambda,\gamma}(\overline{u}_{\mu,\lambda,\gamma,\omega})-\inf J_{\mu,\lambda,\gamma}\right)\\
 & \le2\left[J_{\mu,\lambda,\gamma}(\overline{u}_{\mu,\lambda,\gamma,\omega})-\inf\overline{J}_{\mu,\lambda,\gamma,\omega}\right]+2\left[\overline{J}_{\mu,\lambda,\gamma,\omega}(u_{\mu,\lambda,\gamma})-\inf J_{\mu,\lambda,\gamma}\right]\\
 & \le2M\lambda\gamma^{d+1}\e^{-\gamma\omega},
\end{align*}
as claimed.
\end{proof}

\section{Proof of Theorem~\ref{thm:maintheorem}}\label{sec:mainthmproof}

In this section we prove \cref{thm:maintheorem}. We first need a result
from \cite{GS15}. Recall the notation $d'$ introduced in \eqref{e.def.dprime}. 
\begin{prop}
\label{prop:glivenkocantelliwinfty}Let $U\subset\mathbf{R}^{d}$
be a bounded, connected domain with Lipschitz boundary. Let $\mu$
be a probability measure on $U$, absolutely continuous with respect to
Lebesgue measure, with density bounded above and away from zero 
on $U$. For every $\alpha\ge 1$, there is a constant $C<\infty$, depending only
on $U$, $\alpha$, and $\mu$, such that the following holds. If $(X_{n})_{n\in\mathbf{N}}$
are independent random variables with law $\mu$, then for every integer $N \ge 1$, 
\[
\mathbf{P}\left(\mathcal{W}_{\infty}\left(\mu,\frac{1}{N}\sum_{n=1}^{N}\delta_{X_{n}}\right)\ge CN^{-1/(d\vee2)}(\log N)^{1/d'}\right)\le CN^{-\alpha}.
\]
\end{prop}

\begin{proof}
For $d\ge2$, this is a restatement of \cite[Theorem~1.1]{GS15}. For
$d=1$, the result can be obtained from the classical Kolmogorov-Smirnov quantitative version of the Glivenko-Cantelli theorem. 
\end{proof}
Now we can prove \cref{thm:maintheorem}. For a measure $\mu$ on $\Rd$ and a Borel set $U$, we denote by $\mu \mres U$ the restriction of $\mu$ to the set $U$. 
\begin{proof}[Proof of \cref{thm:maintheorem}.]
 Recalling \cref{eq:Jbarabbrvdef}, it is clear that if $\gamma$ is
so large that
\begin{equation}
(d+4/3)\gamma^{-1}\log\gamma\le\min_{1 \le \ell \neq \ell'\le L}\dist(U_\ell,U_{\ell'}),\label{eq:mindist}
\end{equation}
then
\begin{equation}
\overline{u}_{\mu_{N}\mressmall U_\ell,\lambda,\gamma}(x)=\overline{u}_{\mu_{N},\lambda,\gamma}(x),\qquad\text{for all }x\in U_\ell,\label{eq:truncatedagreediscrete}
\end{equation}
and similarly
\begin{equation}
\overline{u}_{\mu\mressmall U_\ell,\lambda,\gamma}(x)=\overline{u}_{\mu,\lambda,\gamma}(x),\qquad\text{for all }x\in U_\ell.\label{eq:truncatedagreects}
\end{equation}
Also, we have by the definitions and \cref{prop:limitinggivescentroids}
that there exists $\lambda_{\mathrm{c}}$ such that for every $\lambda \ge \lambda_{\mathrm{c}}$,
\begin{equation}
u_{\mu\mressmall U_\ell,\lambda,\infty}(x)=u_{\mu,\lambda,\infty}(x)=\centroid{\mu}(U_\ell),\qquad\text{for all }x\in U_\ell.\label{eq:agreegammainfty}
\end{equation}
By \cref{eq:agreegammainfty} and \cref{thm:convgamma}, we have
\[
\int_{U_\ell}|\centroid{\mu}(U_\ell)-u_{\mu\mressmall U_\ell,\lambda,\gamma}|^2\,\dif\mu=\int_{U_\ell}|u_{\mu\mressmall U_\ell,\lambda,\infty}-u_{\mu\mressmall U_\ell,\lambda,\gamma}|^2\,\dif\mu\le C\gamma^{-1/3}.
\]
By \cref{eq:truncatedagreects} and \cref{eq:Jbarabbrvtruncatebound},
we have, as long as \cref{eq:mindist} holds,
\[
\int_{U_\ell}\left|\overline{u}_{\mu,\lambda,\gamma}-u_{\mu\mressmall U_\ell,\lambda,\gamma}\right|^2\,\dif\mu=\int_{U_\ell}\left|\overline{u}_{\mu\mressmall U_\ell,\lambda,\gamma}-u_{\mu\mressmall U_\ell,\lambda,\gamma}\right|^2\,\dif\mu\le2M\lambda\gamma^{-1/3}.
\]
Combining the last two displays, we see that 
\[
\int_{U_\ell}\left|\overline{u}_{\mu,\lambda,\gamma}-\centroid{\mu}(U_\ell)\right|^2\,\dif\mu\le C(1+\lambda)\gamma^{-1/3}.
\]
Using \cref{eq:Jbarabbrvtruncatebound} again, this implies that
\begin{equation}
\int_{U_\ell}\left|u_{\mu,\lambda,\gamma}-\centroid{\mu}(U_\ell)\right|^2\,\dif\mu\le C(1+\lambda)\gamma^{-1/3}.\label{eq:continuous-convergence}
\end{equation}
On the other hand, by \cref{prop:glivenkocantelliwinfty}, we have for
each $\ell$ that
\begin{equation}
\mathbf{P}\left(\mathcal{W}_{\infty}\left(\frac{\mu\mres U_\ell}{\mu(U_\ell)},\frac{\mu_{N}\mres U_\ell}{\mu_{N}(U_\ell)}\right)\ge CN^{-1/(d\vee2)}(\log N)^{1/d'}\right)\le CN^{-100}.\label{eq:probbound}
\end{equation}
By \cref{prop:Winftystability}, for each $\ell$ there is an $\infty$-optimal transport plan $\pi_{\ell,N}$ between $\frac{\mu\mressmall U_\ell}{\mu(U_\ell)}$ and $\frac{\mu_N\mressmall U_L}{\mu_N(U_\ell)}$ %
such that, using also \cref{eq:truncatedagreediscrete} and \cref{eq:truncatedagreects}, 
we have
\[
\iint_{U_\ell^2}|u_{\mu,\lambda,\gamma}(x)-u_{\mu_{N},\lambda,\gamma}(\tilde x)|^2\,\dif\pi_{\ell,N}(x,\tilde x)\le C{(\gamma+1)}\mathcal{W}_{\infty}\left(\frac{\mu\mres U_\ell}{\mu(U_\ell)},\frac{\mu_{N}\mres U_\ell}{\mu_{N}(U_\ell)}\right).
\]
Combining this with \cref{eq:continuous-convergence}, we see that
\begin{align*}
\frac{1}{\mu_N(U_\ell)}&\int_{U_\ell}|u_{\mu_{N},\lambda,\gamma}-\centroid{\mu}(U_\ell)|^2\,\dif\mu_{N}\\ & =\iint_{U_\ell^2}|u_{\mu_{N},\lambda,\gamma}(\tilde x)-\centroid{\mu}(U_\ell)|^2\,\dif\pi(x,\tilde x)\\
 & \le C\left({(\gamma+1)}\mathcal{W}_{\infty}\left(\frac{\mu\mres U_\ell}{\mu(U_\ell)},\frac{\mu_{N}\mres U_\ell}{\mu_{N}(U_\ell)}\right)+(1+\lambda)\gamma^{-1/3}\right).
\end{align*}
Now summing over $\ell$ and using \cref{eq:probbound} and the fact that
the term inside the expectation on the left-hand side of \cref{eq:thm1statement} is bounded almost surely,
we obtain \cref{eq:thm1statement}.
\end{proof}

\bibliographystyle{abbrv}
\bibliography{localclustering}
\end{document}